\documentclass{article}

\usepackage{microtype}
\usepackage{graphicx}
\usepackage{subfigure}
\usepackage{booktabs} 
\usepackage{amsmath}
\usepackage{amsthm}
\usepackage{amssymb}
\newcommand{\abs}[1]{\left| #1 \right|}
\newcommand{\floor}[1]{\left\lfloor #1 \right\rfloor}

\newtheorem{theorem}{Theorem}[section]

\usepackage{hyperref}

\usepackage[accepted]{icml2020}

\icmltitlerunning{Improving Positive Unlabeled Learning: Practical AUL Estimation and New Training Method for Extremely Imbalanced Data Sets}

\begin{document}

\twocolumn[
\icmltitle{Improving Positive Unlabeled Learning: Practical AUL Estimation and New Training Method for Extremely Imbalanced Data Sets}




\begin{icmlauthorlist}
\icmlauthor{Liwei Jiang}{to}
\icmlauthor{Dan Li}{to}
\icmlauthor{Qisheng Wang}{to}
\icmlauthor{Shuai Wang}{to}
\icmlauthor{Songtao Wang}{to}
\end{icmlauthorlist}

\icmlaffiliation{to}{Department of Computer Science, Tsinghua University, Beijing, China}
\icmlcorrespondingauthor{Liwei Jiang}{jlw17@mails.tsinghua.edu.cn}
\icmlcorrespondingauthor{Dan Li}{tolidan@tsinghua.edu.cn}
\icmlcorrespondingauthor{Qisheng Wang}{wqs17@mails.tsinghua.edu.cn}
\icmlcorrespondingauthor{Shuai Wang}{s-wang17@mails.tsinghua.edu.cn}
\icmlcorrespondingauthor{Songtao Wang}{wangst12@mails.tsinghua.edu.cn }

\icmlkeywords{Machine Learning, ICML}

\vskip 0.3in
]



\printAffiliationsAndNotice{}  

\begin{abstract}
Positive Unlabeled (PU) learning is widely used in many applications, where a binary classifier is trained on the data sets consisting of only positive and unlabeled samples. In this paper, we improve PU learning over state-of-the-art from two aspects. Firstly, existing model evaluation methods for PU learning requires ground truth of unlabeled samples, which is unlikely to be obtained in practice. In order to release this restriction, we propose an asymptotic unbiased practical AUL (area under the lift) estimation method, which makes use of raw PU data without prior knowledge of unlabeled samples.

Secondly, we propose ProbTagging, a new training method for extremely imbalanced data sets, where the number of unlabeled samples is hundreds or thousands of times that of positive samples. ProbTagging introduces probability into the aggregation method. Specifically, each unlabeled sample is tagged positive or negative with the probability calculated based on the similarity to its positive neighbors. Based on this, multiple data sets are generated to train different models, which are then combined into an ensemble model. Compared to state-of-the-art work, the experimental results show that ProbTagging can increase the AUC by up to $10\%$, based on three industrial and two artificial PU data sets.
\end{abstract}

\section{Introduction}
Positive unlabeled (PU) learning is a variant of positive negative (PN) learning and has a wide range of applications, where negative samples are unavailable. For example, in the field of credit card fraud detection \cite{26}, the bank can only obtain partial fraud cases through user complaints, but the other unlabelled users may also include fraudulent ones; in the field of e-commerce product recommendation \cite{25}, a user's favorite products are known according to the user's shopping cart, but it is difficult to obtain the products that the user does not like; in the field of network operation and maintenance, only a few abnormal events can be observed, but there are still many abnormal events that have not been found. In these cases, traditional semi-learning algorithms \cite{21,22,23,24} and supervised learning algorithms \cite{15,16,17} cannot be applied because of the absence of labeled negative samples.

In this work, we improve PU learning over state-of-the-art from two aspects, namely, model evaluation and model training. We not only propose a practical AUL estimation method without requiring prior knowledge of unlabeled samples, but also design a new method to improve the performance of extremely imbalanced PU learning, called ProbTagging.

Model evaluation plays an important role in PU learning, with the help of which we can validate a model and choose the best one among candidate models. However, existing evaluation methods either require class prior knowledge of unlabeled samples~\cite{10,32} or directly use PN labeled data as the testing data set~\cite{2,4,7}. In practice, it is very difficult, if not impossible at all, to obtain either ground truth of unlabeled samples or PN labeled data. To overcome this problem, we attempt to evaluate PU models using raw PU data sets only (without any class prior knowledge of the unlabeled samples). The major challenge is that the commonly used evaluation metrics are computed using fully labeled PN data sets rather than PU data sets. To solve this issue, we propose a practical AUL estimation method for PU learning, where only raw PU data sets are required. We show that the AUL computed with a PU data set is an asymptotic unbiased estimation for that computed with the corresponding PN data set.

In many practical PU learning applications, positive samples and unlabeled ones are extremely imbalanced, say, the number of positive samples is less than 5\% of total samples. Examples include credit card fraud detection and e-commerce recommendation. There is much research work to discuss PU learning \cite{2,4,7,10,32}. Most of them consider to use unlabeled samples to generate negative samples \cite{2,4,7}, or use unlabeled samples to estimate the risk of negative samples \cite{10,32}. But serious discussions of PU learning under extremely imbalanced data are insufficient. To this end, we design a new training method, called \emph{ProbTagging}.

ProbTagging is designed for PU learning with extremely imbalanced data sets and its key idea is to find as many positive samples as possible in unlabeled samples to mitigate the data imbalance. ProbTagging generates multiple PN data sets (even if not completely reliable) from a PU data set by labeling each unlabeled sample positive or negative with a certain probability according to the similarity to other labeled positive samples. The generated PN data sets contain more positive samples compared with raw PU data sets.
After generating multiple PN data sets, several models are trained using those generated PN data sets, which are then aggregated to obtain a final model. Compared with state-of-the-art work, the experimental results show that ProbTagging can increase the AUC by up to 10\%, based on three industrial and two artificial PU data sets.

The rest of this paper is organized as follows: Section \ref{relatedwork} discusses the background and related work. Section \ref{aul} presents the AUL estimation method and theoretical proofs. Section \ref{ProbTag} and Section \ref{experiment} describes the ProbTagging training method as well as the comparison results with baseline solutions, respectively.
Finally, Section \ref{sec6} concludes the paper.

\section{Background and Related Work}\label{relatedwork}
In this section, we formulate the problem setting, introduce notations used in this paper, and review related work of PN learning and PU learning.
\subsection{Problem Settings}
Let $\vec x \in \mathbb{R}^d$ and $y \in \{+1, -1\}$ be random variables from some distribution $\mathcal{D}$, where $d \in \mathbb{N}$ denotes the dimension of $\vec x$. Let $p(\vec x, y)$ be the joint probability density of $(\vec x, y)$, $p_\text{P}(\vec x) = p(\vec x | y = +1)$ and $p_\text{N}(\vec x) = p(\vec x | y = -1)$ the marginal probability density. Let $\theta_\text{P} = \Pr_{(\vec x, y) \sim \mathcal{D}}[y = +1]$ and $\theta_\text{N} = \Pr_{(\vec x, y) \sim \mathcal{D}}[y = -1]$ be the class prior probabilities for the positive and negative classes with $\theta_\text{P} + \theta_\text{N} = 1$. Moreover, we denote $\mathcal{D}_\text{P}$ and $\mathcal{D}_\text{N}$ to be the distribution of positive and negative samples defined by $p_\text{P}(\vec x)$ and $p_\text{N}(\vec x)$, respectively.
A raw set of samples is independent and identically distributed from $\mathcal{D}$ as $\mathcal{X} = \{(\vec x_i, y_i) \}_{i=1}^{n}$.

Let $g: \mathbb{R}^d \to \mathbb{R}$ be an arbitrary real-valued decision function for binary classification, which produces the cumulative function of $g(\vec x)$ with $\vec x$ selected from a distribution $\mathcal{D}$:
\[
    F(\xi|g, \mathcal{D}) = \Pr_{x\sim\mathcal{D}}[g(\vec x) \leq \xi].
\]
Especially, the empirical cumulative function of $g(\vec x)$ with $\vec x$ uniformly selected from a set $\mathcal{X}$ of samples is denoted by
\[
    F(\xi|g, \mathcal{X}) = \frac 1 n \sum_{i=1}^n \mathbb{I}(g(\vec x_i) \leq \xi).
\]

In binary classification, a threshold $s$ is set to distinguish negative and positive samples. More precisely, those samples $\vec x$ with $g(\vec x) \geq s$ are classified by $g$ as positive samples.
For $0 < p < 1$, the threshold is chosen to be
\[
s_p(g, \mathcal{X}) = \inf \{ x: F(x|g, \mathcal{X}) \geq 1-p \},
\]
which is also known as the $p$-th quantile on $\mathcal{X}$.
The value of $s_p(g, \mathcal{X})$ means that there are $\floor{p|\mathcal{X}|}$ samples $\vec x$ with $g(\vec x)$ greater than or equal to $s_p(g, \mathcal{X})$, therefore, samples with a ratio of $p$ are classified by $g$ as positive samples.
\subsection{PN Learning}
In the PN learning, two sets $\mathcal{X}_\text{P}$ and $\mathcal{X}_\text{N}$ of data are obtained from the raw set $\mathcal{X}$ by $\mathcal{X}_\text{P} = \{ \vec x_i: y_i = +1, 1 \leq i \leq n \}$ and $\mathcal{X}_\text{N} = \{ \vec x_i: y_i = -1, 1 \leq i \leq n \}$, respectively. The raw set $\mathcal{X}$ is called a PN data set in the sense of PN learning. The sensitivity of $g$ on a raw set $\mathcal{X}$ of data is denoted as
\[
    \alpha(s|g, \mathcal{X}) = \frac 1 {\abs{\mathcal{X}_\text{P}}} \sum_{\vec x \in \mathcal{X}_\text{P}} \mathbb{I}(g(\vec x) \geq s).
\]

Basically all supervised learning is PN learning in the case of binary classification, e.g. logistic regression \cite{34}, SVM (support vector machine) \cite{14} and deep neural network \cite{15,16,17}.
\subsection{PU Learning}
PU learning is a kind of the classification learning in the case that we have only unlabeled samples and some distinguished positive samples.

In the PU learning, the raw set $\mathcal{X}$ of data is observed to be an observed set $\mathcal{\hat X} = \{(\vec x_i, \hat y_i) \}_{i=1}^{n}$, and two sets $\mathcal{\hat X}_\text{P}$ and $\mathcal{\hat X}_\text{U}$ of data called the \textit{positive}(P) and \textit{unlabeled}(U) data are obtained from the observed set $\mathcal{\hat X}$ by $\mathcal{\hat X}_\text{P} = \{ \vec x_i: \hat y_i = 1, 1 \leq i \leq n \}$ and $\mathcal{\hat X}_\text{U} = \{ \vec x_i: \hat y_i = 0, 1 \leq i \leq n \}$, respectively. The raw set $\mathcal{\hat X}$ is called a PU data set in the sense of PU learning.

In our case, only part of positive samples are observed to be positive, while the rest samples (the rest of positive samples and all negative samples) remain unknown. We use $\hat y$ to denote the observed tag for every $(\vec x, y)$ with $\hat y = 1$ for observed and $\hat y = 0$ for unknown. We assume that
\begin{enumerate}
  \item A positive sample is independently observed to be positive with probability $\theta_\text{O}$ and remains unknown with probability $1-\theta_\text{O}$, i.e. $\Pr[\hat y=1 | y = +1] = \theta_\text{O}$ and $\Pr[\hat y=0 | y = +1] = 1-\theta_\text{O}$.
  \item All negative samples remain unknown with certainty, i.e. $\Pr[\hat y=0 | y = -1] = 1$.
\end{enumerate}
The sensitivity of $g$ on an observed set $\mathcal{\hat X}$ of data is denoted as
\[
    \alpha(s|g, \mathcal{\hat X}) = \frac 1 {\abs{\mathcal{\hat X}_\text{P}}} \sum_{\vec x \in \mathcal{\hat X}_\text{P}} \mathbb{I}(g(\vec x) \geq s).
\]

Practical PU learning includes two aspects: model evaluation and model training.

In many scenes of PU model evaluation \cite{2,4,7,8}, there is lack of PN data sets for testing, in which case we have only an observed data set $\mathcal{\hat X}$ rather than a ground-truth data set to evaluate a model. Some risk estimators \cite{10,32} can be applied to evaluate the models using PU test sets, however, the class prior knowledge is also required. \cite{33} conducts unbiased AUC risk estimation for semi-supervised learning but not for PU learning.

In the field of PU model training, some work \cite{1,2} directly uses unlabeled samples as weighted negative samples and trains them together with positive samples to get a binary classifier. Work in \cite{3,4,5} uses the two-step method, which performs two steps as follows: firstly identifies negative examples from the unlabeled examples, and secondly builds a classifier to classify rest of unlabeled samples iteratively. The most important difference of these two-step methods is the use of different algorithms in the first step. The BaggingPU method \cite{7} uses bagging classifiers for PU learning. Several binary classifiers of BaggingPU are trained with all positive samples (work as positive samples) and some samples randomly extracted from the unlabeled samples (work as negative samples), which are then aggregated to the final model. \cite{9} first shows that PU learning can be solved by cost-sensitive learning, and the intrinsic bias can be eliminated by some non-convex loss functions. Later, a convex formulation for PU classification proposed by \cite{35} can also cancel the bias. The work in \cite{10} improves the unbiased risk estimators \cite{9,35} and proposes nnPU (non-negative PU learning) for PU learning. In addition, \cite{32} first proposes the risk estimator for AUC optimization of PU data. The state-of-the-art risk estimators \cite{10,32} do own a beautiful theoretical foundation, but when observed positive samples are rare, these estimators may not have very accurate estimation of negative samples' risk.

\section{AUL Estimation without Class Prior Knowledge} \label{aul}
Below we provide a theorem that establishes the relationship between the observed sensitivity $\alpha(s|g, \mathcal{\hat X})$ and latent ground true sensitivity $\alpha(s|g, \mathcal{X})$. Based on this, we give an AUL estimation for PU testing.

\textit{Lift curve} \cite{29} is a variant of the ROC curve. For a lift curve of $g$ on $\mathcal{X}$, its abscissa is $p$, and ordinate is $\alpha(s_p|g, \mathcal{X})$.
The area under lift curve is called AUL. For a PN data set, it can be calculated by
$\text{AUL}_\text{PN} = \int_{0}^{1} \alpha(s_p|g, \mathcal{X}) \mathrm{d}p.
$
\cite{29} proves that
$ \text{AUC} -  \text{AUL}_\text{PN}\ = \theta_\text{P}(\text{AUC} - 0.5).$
It implies that $\text{AUL}_\text{PN}$ also can decide whether a model is better than another. In PU testing, we approximate $\text{AUL}_\text{PN}$ of $g$ on $\mathcal{X}$ by measuring the $\text{AUL}_\text{PU}$ of $g$ on $\mathcal{\hat X}$, i.e. $\text{AUL}_\text{PU} = \int_{0}^{1} \alpha(s_p|g, \mathcal{\hat X}) \mathrm{d}p$, and then evaluate models. Theorem \ref{thm:main} explains that the gap between $\text{AUL}_\text{PU}$ and $\text{AUL}_\text{PN}$ becomes very small when $n$ is sufficiently large.

\begin{theorem} \label{thm:main}
For every $0 < p < 1$, let $s_p = s_p(g, \mathcal{X})$. If $\xi_p$ is the unique solution $\xi$ of $F(\xi-|g, \mathcal{D}) \leq 1-p \leq F(\xi|g, \mathcal{D})$, and $F(\xi|g, \mathcal{D}_\textit{P})$ is continuous at $\xi_p$, then for every $\epsilon > 0$, we have
\begin{equation}\label{s1}
\begin{aligned}
    \Pr \left[ \abs{\alpha(s_p|g, \mathcal{X}) - \alpha(s_p|g, \mathcal{\hat X})} > \epsilon \right] & \leq 4e^{-2n\gamma^2} \\ + 4 \left[1-\theta_\text{P}\theta_\text{O}\left(1-e^{-\epsilon^2/18}\right)\right]^n,
\end{aligned}
\end{equation}
where $\gamma = \min\{F(\xi_p + \delta|g, \mathcal{D}) - (1-p), (1-p) - F(\xi_p - \delta|g, \mathcal{D})\} > 0$ for some $\delta > 0$ with $F(\xi_p + \delta|g, \mathcal{D}_\text{P}) - F((\xi_p - \delta)-|g, \mathcal{D}_\text{P}) \leq \epsilon /6 $.

Moreover, if $F(\xi| g, \mathcal{D})$ and $F(\xi| g, \mathcal{D}_\text{P})$ are strictly increasing with respect to $\xi$, then
\begin{equation} \label{s2}
\lim_{n\to\infty} \Pr[\text{AUL}_\text{PU} = \text{AUL}_\text{PN}] = 1.
\end{equation}
\end{theorem}
\begin{proof}
It can be easily seen that Eq (\ref{s1}) implies Eq (\ref{s2}), and we only prove Eq (\ref{s1}) here.

We make the left hand side bounded by the summation of four terms, i.e.
\[
\Pr [| \alpha(s_p|g, \mathcal{X}) - \alpha(s_p|g, \mathcal{\hat X}) | > \epsilon ] \leq q_1+q_2+q_3+q_4,
\]
where
\begin{align*}
 q_1 & = \Pr \left[\abs{ 1-\alpha(s_p|g, \mathcal{X}) - F(\xi_p|g,\mathcal{X}_\text{P})}  > \epsilon_1 \right], \\
 q_2 &= \Pr [| 1-\alpha(s_p|g, \mathcal{\hat X}) - F(\xi_p|g,\mathcal{\hat X}_\text{P})|  > \epsilon_2 ], \\
 q_3 &= \Pr [| F(\xi_p|g,\mathcal{X}_\text{P}) - F(\xi_p|g,\mathcal{D}_\text{P}) |  > \epsilon_3 ], \\
 q_4 &= \Pr [| F(\xi_p|g,\mathcal{\hat X}_\text{P}) - F(\xi_p|g,\mathcal{D}_\text{P}) |  > \epsilon_4 ],
\end{align*}
and $\epsilon_1 + \epsilon_2 + \epsilon_3 + \epsilon_4 = \epsilon$.

We select $q_1$ for example. We introduce a trade-off $\delta_1$ on $\abs{s_p-\xi_p}$. It is not difficult to see that
\begin{align*}
 q_1 \leq &
    \Pr\left[\abs{s_p - \xi_p} > \delta_1 \right] \\ &+ \Pr\left[ \frac{1}{\abs{\mathcal{X}_\text{P}}}\sum_{\vec x \in \mathcal{X}_\text{P}}\mathbb{I}(\abs{g(\vec x) - \xi_p}\leq \delta_1) > \epsilon_1 \right].
\end{align*}
According to the property of the distribution and sample $p$-th quantiles \cite{28}, we have
\[\Pr\left[\abs{s_p - \xi_p} > \delta_1\right] \leq 2e^{-2n\gamma_1^2},\]
where
\begin{equation} \label{eq1}
\begin{aligned}
\gamma_1 = \min\{ & F(\xi_p + \delta_1 |g,\mathcal{D})-(1-p), \\ & (1-p) - F(\xi_p - \delta_1 |g,\mathcal{D})\}.
\end{aligned}
\end{equation}
By Hoeffding's Theorem \cite{28}, we have
\[
\Pr\left[\sum_{\vec x \in \mathcal{ X}_P}\mathbb{I}(\abs{g(\vec x) - \xi_p}\leq \delta_1) > k \epsilon_1 \right] \leq e^{-2k(\epsilon_1-\zeta_1)^2},
\]
where
\begin{equation} \label{eq2}
\zeta_1 = F(\xi_p+\delta_1|g,\mathcal{D}_\text{P}) - F((\xi_p-\delta_1)-|g,\mathcal{D}_\text{P}).
\end{equation}
Finally, we have
\begin{align*}
q_1 & \leq 2e^{-2n\gamma_1^2} + \sum_{k=0}^n \Pr\left[\abs{\mathcal{ X}_\text{P}} = k\right] e^{-2k(\epsilon_1-\zeta_1)^2} \\
& = 2e^{-2n\gamma_1^2} + \sum_{k=0}^n \binom{n}{k} \theta_\text{P}^k (1-\theta_\text{P})^k e^{-2k(\epsilon_1-\zeta_1)^2} \\
& = 2e^{-2n\gamma_1^2}+\left[1-\theta_\text{P}\left(1-e^{-2(\epsilon_1 - \zeta_1)^2}\right)\right]^n.
\end{align*}
Similarly for $q_2$, after introducing trade-off $\delta_2$ on $\abs{s_p-\xi_p}$, we have
\[
q_2 \leq 2e^{-2n\gamma_2^2}+\left[1-\theta_\text{P}\theta_\text{O}\left(1-e^{-2(\epsilon_2 - \zeta_2)^2}\right)\right]^n,
\]
where $\gamma_2$ and $\zeta_2$ are defined for $\delta_2$ by Eq (\ref{eq1}) and (\ref{eq2}) replacing the indices. Again by Hoeffding's Theorem and similar techniques, we can furthermore obtain that
\begin{align*}
q_3 & \leq \left[1 - \theta_\text{P}\left(1-e^{-2\epsilon_3^2}\right) \right]^n, \\
q_4 & \leq \left[1 - \theta_\text{P}\theta_\text{O}\left(1-e^{-2\epsilon_4^2}\right) \right]^n.
\end{align*}
Therefore, the proof is completed by combining with the above inequalities with $\epsilon_1 = \epsilon_2 = \epsilon/3$, $\epsilon_3 = \epsilon_4 = \epsilon/6$ and $\delta_1 = \delta_2$ such that $\zeta_1 = \zeta_2 = \epsilon/6$.
\end{proof}
We note that Theorem \ref{thm:main} implies that in the general case, the values of observed and latent ground true sensitivities are expected to coincide when the number of samples is sufficiently large.
In other words, $\alpha(s_p|g, \mathcal{\hat X})$ is an asymptotic unbiased estimation for $\alpha(s_p|g, \mathcal{X})$, and the so is $\text{AUL}_\text{PU}$ for $\text{AUL}_\text{PN}$.

A raw PU test set can be directly used to calculate $\alpha(s_p|g, \mathcal{\hat X})$ and $\text{AUL}_\text{PU}$ according to prediction probability given by the model. In Theorem \ref{thm:main}, we prove the upper bound of the estimation error of $\alpha(s_p|g, \mathcal{\hat X})$ , which has a negative exponential relation with the number of samples. When $n$ goes to infinity, the probability that $\text{AUL}_\text{PU}$ and $\text{AUL}_\text{PN}$ are equal is 1. This shows that using $\text{AUL}_\text{PU}$ to estimate $\text{AUL}_\text{PN}$  is feasible.

According to the relationship $\text{AUC} -  \text{AUL}_\text{PN}\ = \theta_\text{P}(\text{AUC} - 0.5)$ and Theorem \ref{thm:main}, $\text{AUL}_\text{PU}$ is a valid replacement for AUC on PU model evaluation. It is not necessary to have knowledge of class prior probability $\theta_\text{P}$ when choosing the best one among candidate models.
\section{Design of ProbTagging} \label{ProbTag}
In this section, we introduce the idea of ProbTagging, and present a specific algorithm of ProbTagging.
\subsection{Overview}
ProbTagging is designed for PU learning with extremely imbalanced data sets, where the number of positive samples is less than 5\% of the overall samples. The cases widely exist in many application scenarios, such as credit card fraud detection, e-commerce product recommendation, etc. The key idea of ProbTagging is to find as many positive samples as possible in unlabeled samples to mitigate the data imbalance, thus ProbTagging neither treats unlabeled samples as weighted negative samples \cite{2}, nor directly generates many new data sets that are considered to be negative data sets by sampling uniformly from unlabeled samples \cite{7}.

Specifically, ProbTagging takes the following ideas.
\begin{enumerate}
\item For each unlabeled sample, calculate the similarity between it and positive samples,
     and then tag it positive or negative with probability according to the similarity;
\item Repeat Step 1 for $m$ times, and obtain $m$ PN data sets (even if not completely reliable).
\item Train $m$ models using the $m$ PN data sets obtained in Step 2, and then aggregate them to a final model.
\end{enumerate}

Figure \ref{fig0} shows the schematic diagram of ProbTagging. We can see that new PN data sets have more positive samples, and ProbTagging provides a sample enhancement technique for extremely imbalanced PU learning. ProbTagging is an abstract of training methods. In this way, one can choose different calculation methods of similarity in Step 1 and different models as their base classifiers in Step 3.
\begin{figure*}[t]
\centering
\includegraphics[width=1.8\columnwidth]{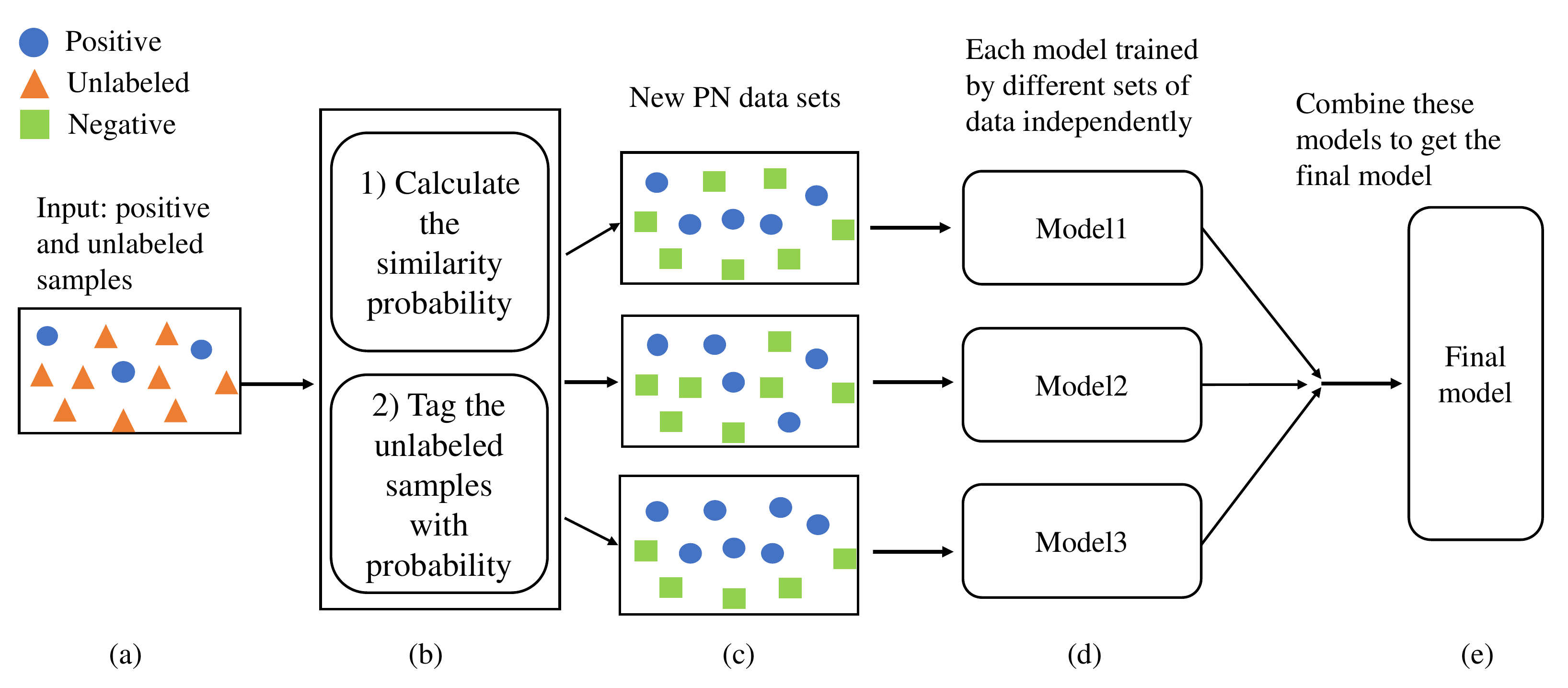} 
\caption{Schematic diagram of the ProbTagging when $m=3$. Blue circles indicate positive samples, orange triangles indicate unlabeled samples, and green squares indicate negative samples.}
\label{fig0}
\end{figure*}

\subsection{Algorithm}
This part provides a specific algorithm of ProbTagging. We suggest that similarity between unlabeled samples and positive samples can be calculated by the $k$-NN algorithm, and give a strategy for selecting the parameter $k$ with the characteristics of the ProbTagging.
\subsubsection{Three Steps of ProbTagging} \label{implement}
We describe the algorithm of each step of ProbTagging three steps.

\textbf{Step 1}. We calculate the $k$-nearest neighbors in the Euclidean distance of each sample in the observed set. Let $N_k(i)$ be the set of $k$-nearest neighbors in $\mathcal{\hat X}$ of sample $\vec x_i$. For each $\vec x_i \in \mathcal{\hat X}_\text{U}$, the similarity $\mathit{Cred}_k(i)$ is selected to be the proportion of positive samples in $N_k(i)$, i.e.
\[
\mathit{Cred}_k(i) = \frac{1}{k}\sum_{j \in N_k(i)}\mathbb{I}(\hat{y}_j = 1).
\]

We tag the unlabeled samples according to $\mathit{Cred}_k(i)$ via uniform distribution. More precisely, for each $\vec x_i \in \mathcal{\hat X}_\text{U}$, let $\beta_i$ uniformly random from $U[0,1]$, where $U[0,1]$ denotes the uniform distribution in the range $[0,1]$. Let \[\hat y_i^{*} = 2 \cdot \mathbb{I}(\beta_i \leq \mathit{Cred}_k(i))-1.\] After tagging all unlabeled samples, we obtain a PN data set $\mathcal{\hat X}^{*} = \{ (\vec x_i, \hat y_i^*): 1 \leq i \leq n \}$. Two data sets $\mathcal{\hat X}_\text{P}^{*}$ and $\mathcal{\hat X}_\text{N}^{*}$ are obtained from the PN data set $\mathcal{\hat X}^{*}$ by $\mathcal{\hat X}_\text{P}^{*} = \{ \vec x_i: \hat y_i^{*} = 1, 1 \leq i \leq n \}$ and $\mathcal{\hat X}_\text{U} = \{ \vec x_i: \hat y_i^{*} = -1, 1 \leq i \leq n \}$, respectively.

\textbf{Step 2}. Repeat Step 1 for $m$ times. Let $\mathcal{\hat X}_j^{*}$ be the $j$-th PN data set we obtain. In this implementation, the base classifier is selected to be GBDT \cite{38}. We denote the $\hat g_j^*$ to be model trained by GBDT with $\mathcal{\hat X}_j^{*}$.

\textbf{Step 3}. The final model is defined to be the average value of $\hat g_j^*$. i.e.
\[g(\vec{x}) = \frac 1 m \sum_{j=1}^m \hat{g}_j^*(\vec x).\]

\subsubsection{Parameter Selection}\label{sec-selectk}
For selecting a suitable parameter $k$, let $E_k$ be the expected number of unlabeled samples that are tagged positive. That is, \[E_k = \sum_{i}\mathit{Cred_k}(i)\cdot\mathbb{I}(\hat y_i = 0).\]

We choose the Credit Card Fraud data set in Table \ref{table1} as an illustrative example, the function image of $E_k$ with respect to $k$ is shown Figure \ref{fig3}. (The information of all data sets is later described in Section \ref{dataset}.)

As we can see, the value of $E_k$ changes in two stages as $k$ increases.
\begin{enumerate}
\item Initial stage: $E_k$ increases significantly as $k$ increases.
In this stage, if the value of $k$ slightly increases, a remarkable number of unlabeled samples will be tagged positive, which makes good use of the only positive samples.

\item Stable stage: $E_k$ changes little and decreases slightly as $k$ increases.
In this stage, increasing $k$ brings few benefits. As a trade-off, small $k$ does not make full use of the positive samples, and large $k$ requires too much calculation.

\end{enumerate}

Therefore, suitable values of $k$ are chosen around the interchange of the two stages. For example, the value of $k$ can be between 16 and 20 in Figure \ref{fig3}.

\begin{figure}[t]
\centering
\includegraphics[width=0.98\columnwidth]{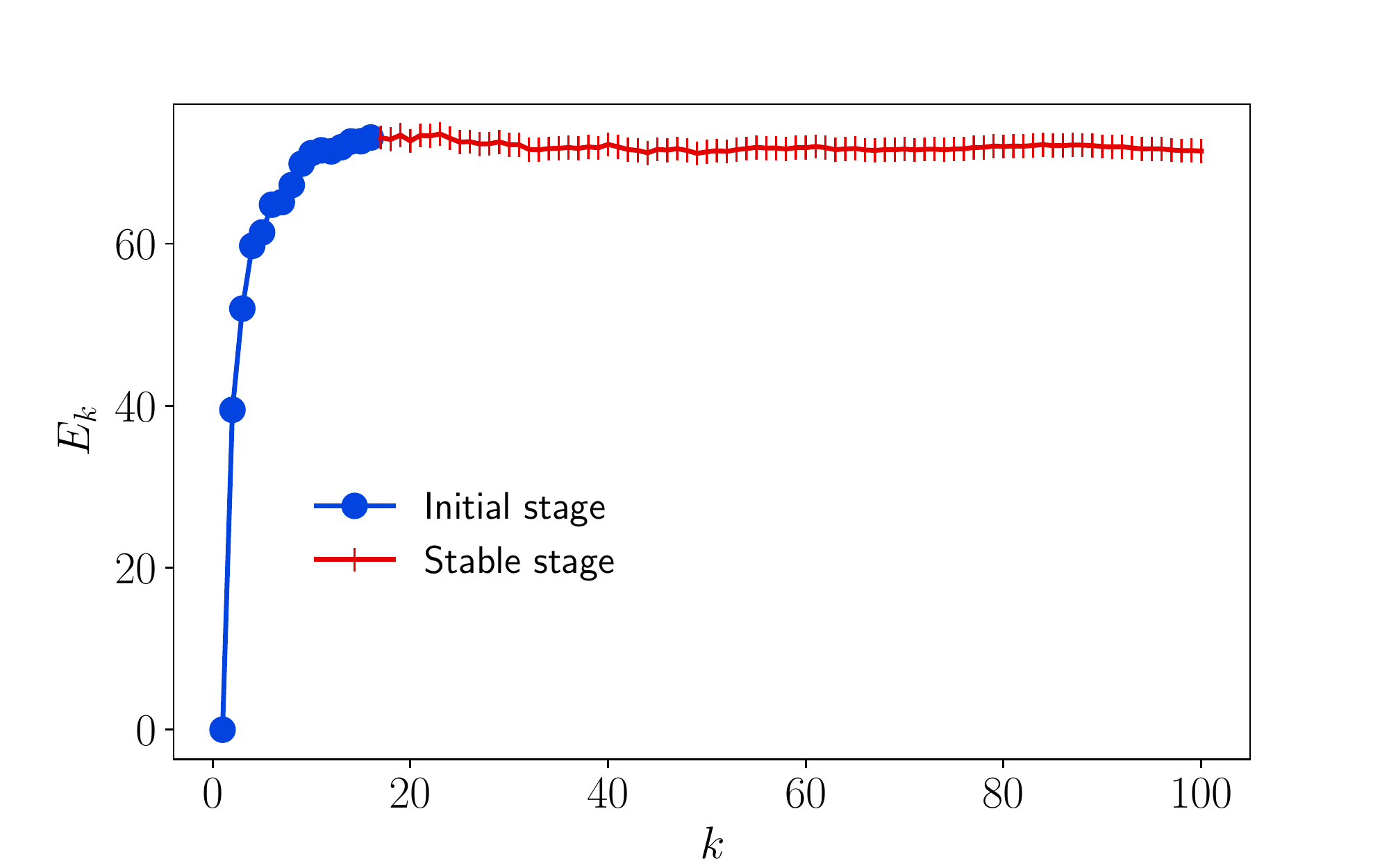} 
\caption{The image of $E_k$ with respect to $k$ induced by a training set of some fold of Credit Card Fraud data set.}
\label{fig3}
\end{figure}

\section{Experimental Evaluation}\label{experiment}
In this section, we first introduce the preparation of data sets and the evaluation method of the experiments. Finally, we show the performance of ProbTagging on different data sets compared to several other PU learning methods.
\subsection{Experimental Setting}
\subsubsection{Data Sets} \label{dataset}
Five data sets are used in our experiment, three of which are real-world raw PU data sets in the bank scenario, and the rest two are from UCI and kaggle databases. See Table \ref{table1} for more information about data sets. As we can see in Table \ref{table1}, these data sets are all extremely imbalanced.

The real-world data sets Bank1, Bank2 and Bank3 from bank scenario are used for distinguishing those credit card customers that have demands for savings cards. The acquisition process of the data sets is set up in two stage. In the first stage, the bank can obtain the features of the credit card customers; and in the second stage, the credit card customers who have processed the savings card at this stage are considered as positive samples. The rest of credit card customers do not have a savings card in the designated observation period (the second stage), but it does not mean that these customers do not need savings cards at all, therefore the labels of these credit card customers remain unknown. We regard those customers as unlabeled samples. Clearly, it can be seen that PU learning fits in the bank scenario.

The data sets APS from UCI \cite{Dua:2019} and Credit Card Fraud from kaggle \cite{credit-kaggle} are PN data sets for binary classification. They need to be manually manipulated to PU data sets by selecting a proportional parameter $\theta_\text{O}$, which indicates that $\theta_\text{O}$ of positive samples are treated as observed positive samples and the remaining positive samples and all negative samples are treated as unlabeled samples. Different data sets can be generated as the value of $\theta_\text{O}$ varies.

In the previous work \cite{2,4,7,10,32}, the experimental data were all PN data from public databases. But in real PU learning scenarios, negative samples are absent. Therefore, in order to make the research work closer to the real PU learning scenario, we use raw PU data sets such as Bank1, Bank2 and Bank3 to perform model training and evaluation for the first time.
\begin{table}[t]
\caption{Data sets description.}\smallskip
\centering
\resizebox{.98\columnwidth}{!}{
\smallskip\begin{tabular}{l|l|l|l}
\hline
\textbf{Data sets}  & \textbf{\#Ins.} & \textbf{\#Attr.} & \textbf{\#Positive ins.} \\ \hline
Bank1   & 1365550         & 30               & 5173                     \\ \hline
Bank2   & 1365550         & 30               & 4989                     \\ \hline
Bank3   & 1365550         & 30               & 5154                     \\ \hline
APS               & 16000           & 169              & 375                      \\ \hline
Credit Card Fraud & 284807          & 30               & 492                      \\ \hline
\end{tabular}
}
\label{table1}
\end{table}

\subsubsection{Evaluation Method}
We use different methods to evaluate the models for data sets from different scenes.

For a PU data set (Bank1, Bank2 and Bank3), we divide it into a training set and a test set, the PU model is trained with the training set and evaluated with the test set by AUL estimation method. We recall that in Section \ref{aul}, $\text{AUL}_\text{PU}$ is proved to be a reasonable metric in evaluating PU models using PU data sets.

For a PN data set (APS and Credit Card Fraud), we divide it into a PN training set and a PN test set. The PN training set is manually manipulated into a PU training set by selecting parameter $\theta_\text{O}$. Then the model is trained with the PU training set. Then the model is trained with the PU training set, and AUC is computed using the PN test set. In order to compare the two evaluation metrics AUC and $\text{AUL}_\text{PU}$, the PN test set is manually manipulated into a PU test set with the same parameter $\theta_\text{O}$, and then evaluate the PU model with the PU test set by AUL estimation method.

We compare our method against five PU learning methods: Elkan's method \cite{2}, Liu's method \cite{4}, BaggingPU \cite{7}, nnPU \cite{10} and PU-AUC \cite{32}.

For fair comparison, we use GBDT as classifier of Elkan's method, Liu's method, BaggingPU and ProbTagging in the experiments. To accelerate training, we use GPU version xgboost \cite{36} for Bank1, Bank2 and Bank3 and lightgbm \cite{37} for APS and Credit Card Fraud as their specific implementation methods of GBDT. In particular, both GPU version xgboost and lightgbm use their default parameters.

The implementation used for ProbTagging is described in Section \ref{implement}. the parameter $k$ is selected according to \ref{sec-selectk}. ProbTagging and BaggingPU both use $m = 50$ base classifiers. We use the open source code of \cite{10} to implement nnPU. There are three versions of nnPU: linear model, 3-layerfully-connected neural network and multi-layer fully-connected neural network in the open source code. We directly use the network structure and parameters given in the open source code and choose the best result of the three versions as the final result. In addition, we follow the implementation details of \cite{32} to implement PU-AUC.

Class prior knowledge is required in nnPU and PU-AUC. For APS and Credit Card Fraud, the class prior probabilities are set to the actual values in their original PN data. For Bank1, Bank2 and Bank3, since class prior knowledge is unknown and impossible to be obtained, it is reasonable to set the positive-class prior probability to the proportion of positive samples in original PU data sets.

We use cross-validation for each data set, and take the average result of each fold as the final result. For Bank1, Bank2, Bank3 and Credit Card Fraud, we use 3-fold cross-validation; for APS, we use 5-fold cross-validation.

\subsection{Results} \label{sec-result}
\subsubsection{Results of Real-world Data Sets}
The three data sets Bank1, Bank2 and Bank3 are real-world industrial PU data sets. Therefore, no PN labels are available for reference. The AUL estimation method proposed in Section \ref{aul} is used to evaluate models.

The results of real-world bank data sets are shown in Table \ref{table2}. We see that ProbTagging performs best on Bank1 and Bank3 and slightly worse than BaggingPU on Bank2. In summary, ProbTagging performs well on real-world extremely imbalanced PU data sets.

\begin{table}[t]
\caption{AUL estimation results of real-world bank data sets.}\smallskip
\centering
\resizebox{.9\columnwidth}{!}{\smallskip
\begin{tabular}{l|l|l|l}
\hline
AUL           & Bank1      & Bank2       & Bank3 \\ \hline
ProbTagging     & \textbf{0.6082} & 0.5971 & \textbf{0.6099}   \\ \hline
Elkan's method & 0.6052          & 0.5943          & 0.6063            \\ \hline
Liu's method  & 0.5884          & 0.5776          & 0.5849            \\ \hline
BaggingPU       & 0.6072          & \textbf{0.5976}          & 0.6086            \\ \hline
nnPU           & 0.5286          & 0.5186          & 0.5202            \\ \hline
PU-AUC  & 0.5455          & 0.5484          & 0.5533
\\ \hline
\end{tabular}
}
\label{table2}
\end{table}

\subsubsection{Results of Public Data Sets}
For a PN data set, we need to manipulate it to PU training set according to the parameter $\theta_\text{O}$. Table \ref{table3} shows the AUC results and AUL estimation results for each data set under different methods with the parameter $\theta_\text{O} = 0.5$. In this case, ProbTagging performs very well on both APS and Credit Card Fraud.

It is worth mentioning that in Table \ref{table3}, the AUL estimation on PU test sets gives the same assertion as the AUC on PN test sets about which model is better. We see that ProbTagging not only performs well under ordinary circumstances but also shows its robustness in the case of high noise.
\begin{table*}[t]
\centering
\caption{AUC and AUL estimation results of public data sets with $\theta_\text{O}$ = 0.5.}\smallskip
\begin{tabular}{|c|c|c|c|c|}
\hline
                          & \multicolumn{2}{c|}{APS}          & \multicolumn{2}{c|}{Credit Card Fraud} \\ \hline
\multicolumn{1}{|l|}{}         & AUC             & AUL             & AUC                & AUL               \\ \hline
ProbTagging              & \textbf{0.9921} & \textbf{0.9878} & \textbf{0.9757}    & \textbf{0.9794}   \\ \hline
Elkan's method          & 0.9896          & 0.9848          & 0.9612             & 0.9681            \\ \hline
Liu's method          & 0.9850          & 0.9806          & 0.9652             & 0.9703            \\ \hline
BaggingPU          & 0.9883          & 0.9837          & 0.9682             & 0.9730            \\ \hline
nnPU         & 0.9233          & 0.9271          & 0.8916             & 0.8858            \\ \hline
PU-AUC & 0.9660          & 0.9653          & 0.8768             & 0.8778            \\ \hline
\end{tabular}
\label{table3}
\end{table*}

Figure \ref{fig1} and \ref{fig2} present the results of different methods on Credit Card Fraud and APS as the value of $\theta_\text{O}$ changes, respectively. We note that a large value of ($1-\theta_\text{O}$) indicates a low quality of the data sets: we have many unlabeled samples but few positive samples. When ($1-\theta_\text{O}$) increases, the experimental results of each method become worse due to the decreasing number of observed positive samples. However, ProbTagging does not have a noticeable decline when ($1-\theta_\text{O}$) is particularly large. We see that ProbTagging not only performs well in smooth cases but also shows its robustness in extremely imbalanced situations.

\begin{figure}[t]
\centering
\includegraphics[width=0.98\columnwidth]{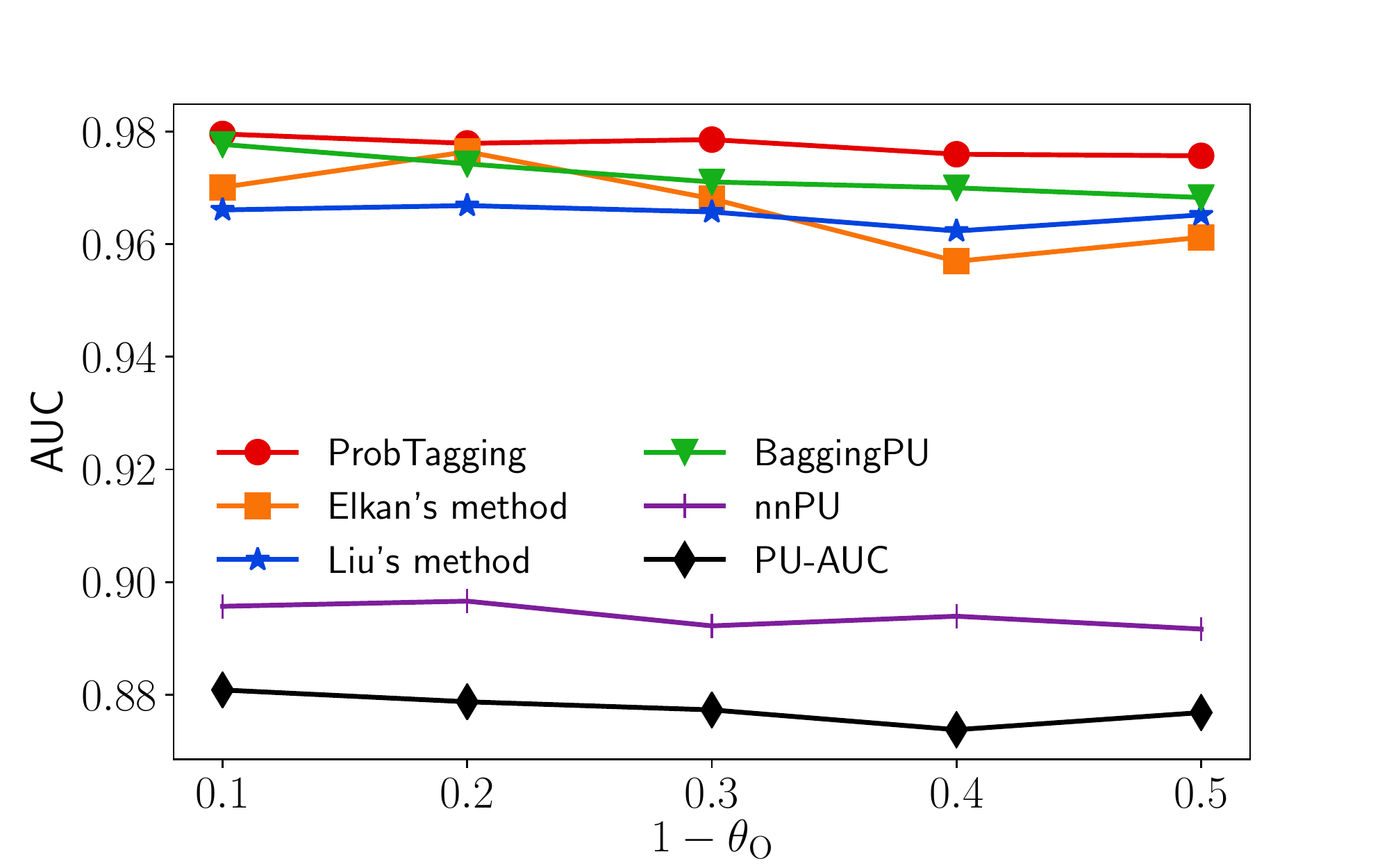} 
\caption{AUC results of Credit Card Fraud data sets.}
\label{fig1}
\end{figure}

\begin{figure}[t]
\centering
\includegraphics[width=0.98\columnwidth]{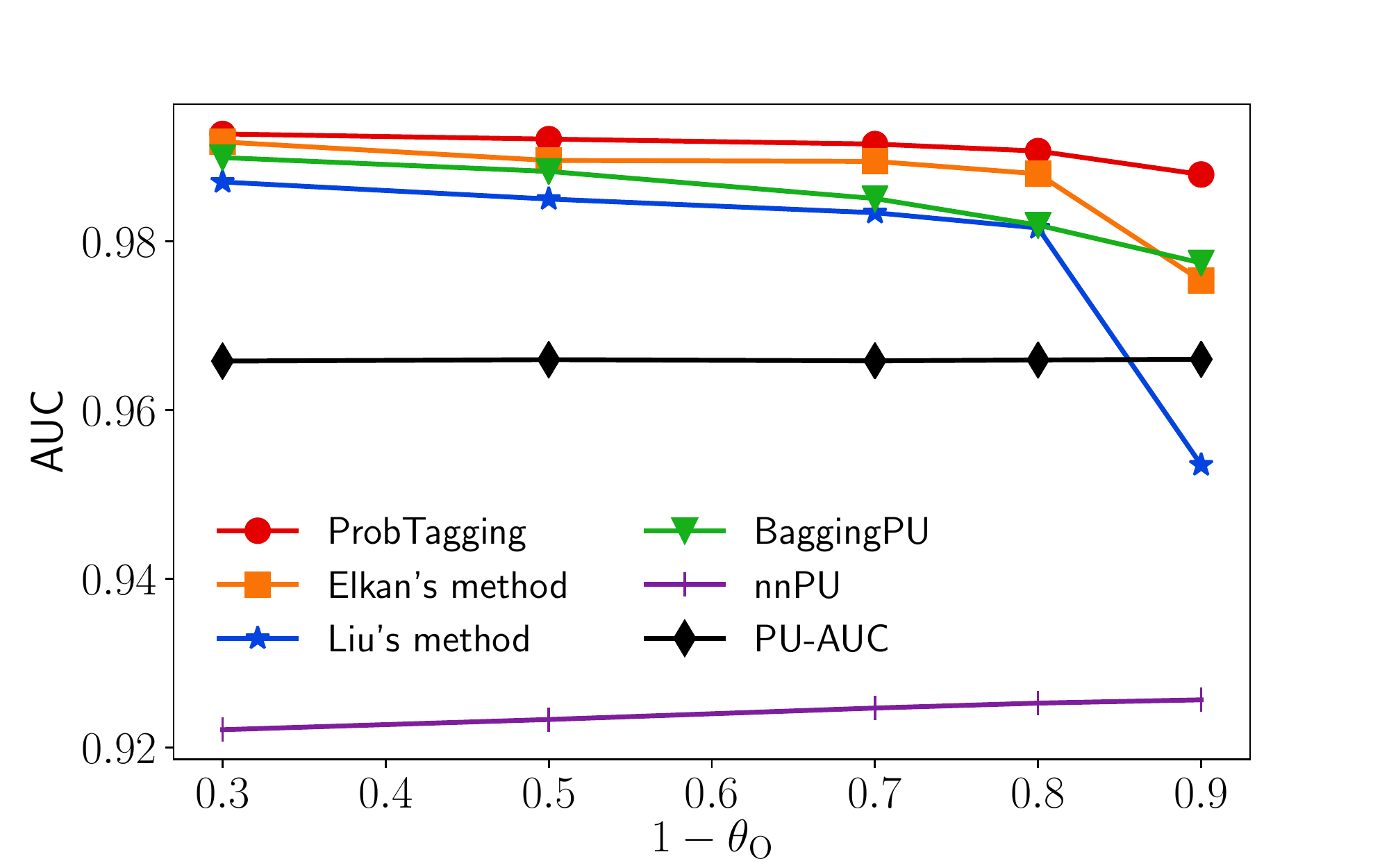} 
\caption{AUC results of APS data sets.}
\label{fig2}
\end{figure}
\subsubsection{Discussions}
Learning in extremely imbalanced PU data sets is very challenging. In the real-world PU data sets from industrial scenes in Table \ref{table2}, all methods do not have very well performance. This is because the industrial data sets tend to be very noisy. However, we see that the improvements from ProbTagging still bring many benefits to actual industrial scenarios.

We can see that nnPU and PU-AUC do not have very well performance on extremely imbalanced data sets, where the positive samples are rare (just a few percent or even a few thousandths of unlabeled samples). This is because when there are very few positive samples, the contribution of positive samples to the risk estimators calculation in nnPU and PU-AUC is very small, which is intended to be significant, and thus the risk of negative samples cannot be estimated correctly.

Each base classifier of BaggingPU uses the same positive samples, therefore they still have a strong correlation, and thus the variance of the final model is not effectively reduced. In the experiments, BaggingPU is worse than the ProbTagging.

The work of \cite{2} proposes that: under the assumption that the observed positive examples are selected randomly from total positive samples, PU model predicts probabilities that differ by only a constant factor from the true conditional probabilities of being positive. Because AUC depends on the sort of test samples, the constant factor does not affect the value of AUC. Therefore, when a data set satisfies the assumption, the PU model trained by Elkan's method is close to the corresponding PN model in terms of AUC. We can see that Elkan's method sometimes has a better performance than BaggingPU in the experiments.

There is a challenge for Liu's method that the initial model needs to be strong enough due to the large correlation between the initial and the final model, it is difficult to find pure negative or positive samples at the first step. In the experiments, the performance of Liu's method is a little worse than BaggingPU and Elkan's method.

\section{Conclusion}\label{sec6}
In this paper, we improve PU learning over state-of-the-art from two aspects: model evaluation and model training. We propose the practical AUL estimation without class prior knowledge. The AUL estimation method not only presents a new idea of model evaluation, but also could be helpful for other learning or mathematical problems in the future.

In addition, we also design a new training method called ProbTagging to improve the performance of extremely imbalanced PU learning. It is noted that ProbTagging is an abstract of learning methods, where the calculation of similarity and the base classifier vary with the learning task. The specific algorithm of ProbTagging in this paper uses $k$-NN for similarity calculation and GBDT as base classifier. Compared with state-of-the-art work, the experimental results show that ProbTagging can increase the AUC by up to 10\%. ProbTagging also provides a sample enhancement technique, and thus we can consider applying ProbTagging to other fields besides PU learning, such as multi-class semi-supervised learning.


\end{document}